\documentclass[twoside]{article}

\usepackage[accepted]{aistats2020}

\usepackage[numbers]{natbib}

\usepackage{lipsum}

\usepackage[utf8]{inputenc} % allow utf-8 input
\usepackage[T1]{fontenc}    % use 8-bit T1 fonts
\usepackage{lmodern}
\usepackage{hyperref}       % hyperlinks
\usepackage{url}            % simple URL typesetting
\usepackage{booktabs}       % professional-quality tables
\usepackage{amsfonts}       % blackboard math symbols
\usepackage{nicefrac}       % compact symbols for 1/2, etc.
\usepackage{microtype}      % microtypography

\usepackage[linesnumbered, ruled, vlined, noend]{algorithm2e}
\usepackage{amsmath}
\usepackage{amsthm}
\usepackage{amssymb}

\usepackage{bm}
\usepackage[title]{appendix}
\usepackage{courier}
\usepackage{enumitem}
\usepackage{graphicx}
\usepackage{algorithmic}
\usepackage[caption=false,font=footnotesize]{subfig}
\usepackage[margin=1in]{geometry}
\usepackage{authblk}
\usepackage{nicefrac}
\usepackage{mathtools}
\usepackage{color}
\usepackage{xspace} % Correct macro spacing
\usepackage{times}
\usepackage{hyperref}
\usepackage{stmaryrd}

\newcommand{\llangle}{\langle\!\langle}
\newcommand{\rrangle}{\rangle\!\rangle}

\usepackage{todonotes}

\newcommand{\fTT}[1]{f_{\mathrm{TT}(#1)}}
\newcommand{\fCP}[1]{f_{\mathrm{CP}(#1)}}

\newtheorem*{theorem*}{Theorem}
\newtheorem*{lemma*}{Lemma}

\RequirePackage{latexsym}
\RequirePackage{amsmath}
\RequirePackage{amssymb}
\RequirePackage{bm}
\RequirePackage[mathscr]{euscript}

\newcommand{\Ex}[1]{\mbox{}\mathbb{E}\left[#1\right]}
\newcommand{\Var}[1]{\mbox{}\textup{Var}\left(#1\right)}

\newcommand{\Tr}[1]{\mbox{}\textup{Tr}\left[#1\right]}

\newcommand{\ts}{\mathsf{T}}
\newcommand{\RR}[2]{\mathbb{R}^{#1 \times #2}} % Real numbers
\newcommand{\vectorize}{\mathrm{vec}}
\newcommand{\bigo}{\mathcal{O}}

\newcommand{\St}{\bm{\mathcal{S}}}
\newcommand{\At}{\bm{\mathcal{A}}}
\newcommand{\Bt}{\bm{\mathcal{B}}}
\newcommand{\Xt}{\bm{\mathcal{X}}}
\newcommand{\Mt}{\bm{\mathcal{M}}}
\newcommand{\Gt}{\bm{\mathcal{G}}}
\newcommand{\Tt}{\bm{\mathcal{T}}}
\newcommand{\It}{\bm{\mathcal{I}}}

\newcommand{\CP}[1]{\llbracket #1 \rrbracket}

\newcommand{\TT}[1]{\llangle #1 \rrangle}

\newcommand{\ab}{\mathbf{a}}
\newcommand{\bb}{\mathbf{b}}

\newcommand{\ub}{\mathbf{u}}
\newcommand{\vb}{\mathbf{v}}

\newcommand{\xb}{\mathbf{x}}
\newcommand{\yb}{\mathbf{y}}

\newcommand{\Ab}{\mathbf{A}}
\newcommand{\Bb}{\mathbf{B}}

\newcommand{\Ib}{\mathbf{I}}

\newcommand{\Sb}{\mathbf{S}}

\newcommand{\Xb}{\mathbf{X}}

\ifx\BlackBox\undefined
\newcommand{\BlackBox}{\rule{1.5ex}{1.5ex}}  % end of proof
\fi

\ifx\QED\undefined
\def\QED{~\rule[-1pt]{5pt}{5pt}\par\medskip}
\fi

\ifx\proof\undefined
\newenvironment{proof}{\par\noindent{\bf Proof\ }}{\hfill\BlackBox\\[2mm]}
\fi

\ifx\theorem\undefined
\newtheorem{theorem}{Theorem}
\fi

\ifx\example\undefined
\newtheorem{example}{Example}
\fi

\ifx\property\undefined
\newtheorem{property}{Property}
\fi

\ifx\lemma\undefined
\newtheorem{lemma}[theorem]{Lemma}
\fi

\ifx\proposition\undefined
\newtheorem{proposition}[theorem]{Proposition}
\fi

\ifx\remark\undefined
\newtheorem{remark}[theorem]{Remark}
\fi

\ifx\corollary\undefined
\newtheorem{corollary}[theorem]{Corollary}
\fi

\ifx\definition\undefined
\newtheorem{definition}{Definition}
\fi

\ifx\step\undefined
\newtheorem{step}{Step}
\fi

\ifx\conjecture\undefined
\newtheorem{conjecture}[theorem]{Conjecture}
\fi

\ifx\axiom\undefined
\newtheorem{axiom}[theorem]{Axiom}
\fi

\ifx\claim\undefined
\newtheorem{claim}[theorem]{Claim}
\fi

\ifx\assumption\undefined
\newtheorem{assumption}[theorem]{Assumption}
\fi

\newcommand{\EE}{\mathbb{E}} % Expectation
\newcommand{\PP}{\mathbb{P}} % Probability
\renewcommand{\RR}{\mathbb{R}} % Real numbers
\newcommand{\tr}{\mathop{\mathrm{tr}}}

\newcommand{\eg}{\emph{e.g.} }
\newcommand{\ie}{\emph{i.e.} }

\newcommand{\nbr}[1]{\left\|#1\right\|}

\newcommand{\abs}[1]{\left|#1\right|}

\newcommand{\inner}[2]{\left\langle #1,#2 \right\rangle}
\newcommand{\norm}[1]{\|#1\|}

\newcommand{\zero}{\mathbf{0}} % Zero
\usepackage{lmodern}

\usepackage[symbol]{footmisc}

\begin{document}
%
%
%
%
%
%
%
%
%
%
%
%
%
%

%
%
%
%
%
% \aistatsauthor{Beheshteh T. Rakhshan and Guillaume Rabusseau}

\twocolumn[

\aistatstitle{Tensorized Random Projections}

\aistatsauthor{ Beheshteh T. Rakhshan \And Guillaume Rabusseau\footnotemark[1] }

\aistatsaddress{ Department of Mathematics, Purdue University \And  DIRO and Mila, Universit\'{e} de Montr\'{e}al } ]

\begin{abstract}
We introduce a novel random projection technique for efficiently reducing the dimension of very high-dimensional tensors. Building upon classical results on Gaussian random projections and Johnson-Lindenstrauss transforms~(JLT), we propose two tensorized random projection maps relying on the tensor train~(TT) and CP decomposition format, respectively. The two maps offer  very low  memory requirements and can be applied efficiently when the inputs are low rank tensors given in the CP or TT format.
Our theoretical analysis shows that the dense Gaussian matrix in JLT can be replaced by a low-rank tensor implicitly represented in compressed form with random factors, while still approximately preserving the Euclidean distance of the projected inputs. In addition, our results reveal that the TT format is substantially superior to CP in terms of the size of the random projection needed to achieve the same distortion ratio. Experiments on synthetic data validate our theoretical analysis and demonstrate the superiority of the TT decomposition.

\end{abstract}
\section{Introduction}

Random projections~(RP) are commonly used in data science and machine learning to project down high-dimensional data into a lower dimensional space while preserving most of the relevant information in the data~\citep{vempala2005random,bingham2001random}. These methods have been successfully used to trade accuracy in order to reduce time and storage complexity of classical learning algorithms such as $k$-nearest neighbors~\citep{ailon2006approximate,ailon2009fast,indyk1998approximate,kleinberg1997two}, $k$-means~\citep{boutsidis2010random}, support vector machines~\citep{paul2013random} and learning high-dimensional Gaussian mixtures~\citep{dasgupta1999learning,dasgupta2013experiments} to name a few. Most modern RP techniques build upon the celebrated \emph{Johnson-Lindenstrauss lemma}~\citep{johnson1984extensions} which shows that an arbitrary number of high-dimensional points can be linearly projected into an exponentially lower dimensional subspace while preserving distances between points. One of the simplest Johnson-Lindenstrauss transforms~(JLT) is constructed from a random matrix $\Ab$ whose entries are independently and identically drawn from a normal distribution. Fast variants of JLT have been proposed by introducing sparsity in $\Ab$~\citep{achlioptas2003database,li2006very} and by leveraging fast matrix multiplication algorithms~\citep{ailon2006approximate,ailon2009fast,ailon2013almost}.

\footnotetext{\footnotemark[1] CIFAR AI Chair}
At the same time, tensor decomposition techniques have also recently emerged as a powerful tool for dealing with high-dimensional data. Tensor methods are particularly suited to handle high-dimensional multi-modal data and have been successfully applied in neuroimaging~\citep{Zhou2013}, signal processing~\citep{Cichocki2009,sidiropoulos2017tensor}, 
 spatio-temporal analysis~\citep{bahadori2014fast} and computer vision~\citep{Lu2013}. But even when the data is not inherently multi-modal in nature, tensor decomposition techniques can be used to speed-up and scale classical learning algorithms to very high-dimensional spaces~\citep{novikov2014putting,novikov2015tensorizing}. Such algorithms exploit the ability of tensor decomposition techniques to implicitly represent very high-dimensional data in compressed form, by first tensorizing the data before applying tensor decomposition techniques. In particular, the CANDECOMP/PARAFAC~(CP)~\citep{hitchcock1927expression} and tensor-train~(TT)~\citep{oseledets2011tensor} decomposition can represent $N$th order  $d$-dimensional tensors~(or equivalently $d^N$-dimensional vectors) using only $\bigo(NdR)$ and $\bigo(NdR^2)$ parameters respectively, where the rank parameter $R$ controls the coarseness of the decomposition. Crucially, the number of parameters of these decomposition only grows \emph{linearly} with the order of the tensor $N$, which is not the case for other popular decomposition models such as the Tucker decomposition~\citep{tucker1966some}.

While efficient random projection techniques have been proposed to deal with high-dimensional data, RP still suffer from the curse of dimensionality when the input dimension is very large, which is the case for high-order tensor inputs. 
In this work, we propose to leverage tensor decomposition techniques to \emph{tensorize} Gaussian random projections. In doing so, we design efficient random projections that can be applied to any high-order tensor inputs with arbitrary rank and structure. In particular, projecting an input tensor given in the CP or TT format can be done very efficiently. More precisely, we propose two tensorized random projection maps, $\fTT{R}$ and $\fCP{R}$, relying on the TT and CP formats respectively. %

Intuitively, the random projection maps $\fTT{R}$ and $\fCP{R}$ are constructed by enforcing a low rank tensor structure~(CP or TT) on the rows of the random projection matrix $\Ab\in\RR^{k\times d^N}$ where $k\ll d^N$ is the size of the random projection and the inputs are $N$th-order $d$-dimensional tensors. The parameter $R$ corresponds to the rank of the CP/TT decomposition used to represent the rows of $\Ab$ and controls the tradeoff between the quality of the embedding and the computational and memory cost of projecting input points. More precisely, if the input $\Xt$ is given as a rank $\tilde{R}$ CP or TT tensor, computing $\fTT{R}(\Xt)$ and $\fCP{R}(\Xt)$ can be done in time $\bigo(kNd\max(R,\tilde{R})^3)$. In terms of memory requirements, $\fTT{R}$ and $\fCP{R}$ have $\bigo(kNdR^2)$ and $\bigo(kNdR)$ parameters respectively. In comparison the cost of transformation for a Gaussian JLT is in $\bigo(kd^N)$ which can be improved to $\bigo(k+Nd^N\log d)$ using fast JLT.

Our theoretical analysis shows that the key properties of Gaussian random projections are preserved after \emph{tensorization}: for any $\varepsilon>0$, with high probability, our tensorized RP embed any set of $m$ points up to multiplicative distortion $(1\pm \varepsilon)$ as soon as $k \gtrsim \varepsilon^{-2}(1+2/R)^N\mathrm{log}^{2N}m$ for $\fTT{R}$ and $k \gtrsim \varepsilon^{-2}3^{N-1}(1+2/R)\mathrm{log}^{2N}m$ for $\fCP{R}$. Besides showing that both tensorizations lead to efficient random projections (in terms of time and memory complexity), our analysis further reveals that $\fTT{R}$ is substantially superior to $\fCP{R}$ in terms of the size of the random projection needed to achieve the same multiplicative distortion. This can be seen by comparing the exponential dependency on the order $N$ of input tensors in the lower bounds on $k$ given above~(and how increasing the rank $R$ of the tensorized map can mitigate this dependency). In particular, our analysis shows that the CP format is not a reasonable decomposition format for tensorizing random projections in the case of high order input tensors.
\textbf{Summary of contributions.}
We present two tensorized random projection maps, $\fTT{R}$ and $\fCP{R}$, using the TT and CP decomposition models respectively. We show that both maps are \textit{Johnson-Lindenstrauss Transforms} offering appealing computational and memory requirements. In particular, our work is the first to design efficient RP for input tensors given in the CP or TT format.  Our theoretical analysis for $\fCP{R}$ extends the one first initiated in~\citep{sun2018tensor}~(which was limited to matrix inputs) to high-order input tensors. To the best of our knowledge, this is the first time that the TT decomposition model is leveraged to design RP that can scale to very high-dimensional inputs. Our theoretical analysis further  shows that the TT format is a better decomposition model than CP for tensorizing random projection maps. Our numerical simulations substantially validate this conclusion. It is worth mentioning that our analysis is not focused on rank-one tensors and holds for  arbitrary input tensors with low CP rank or TT rank structure.

\textbf{Related work.}
Tensor Sketch~\citep{pham2013fast} is an extension of the Count Sketch algorithm~\citep{charikar2002finding} using fast FFT which can efficiently approximate polynomial kernels. More recently, \citep{shi2019multi} extended Tensor Sketch to exploit the multi-modal structure of tensor inputs, but their approach relies on the Tucker decomposition format and cannot scale to very high-order tensors.  Kapralov et al. ~\citep{kapralov2019oblivious} also consider sketching tensor products of data points without
explicitly forming the resulting tensor, and propose an algorithm to compute a linear sketch for degree-$N$ polynomial kernels. 

More closely related to our work, Sun et al.~\citep{sun2018tensor} introduce a Tensor Random Projection map~(TRP) using a row-wise Kronecker product of random matrices. We show that their method is equivalent to the CP tensorized random projection map studied in this paper. Their theoretical analysis is limited to order 2 tensors~(i.e. matrices) and rank one projection maps:  they show that TRP satisfies the JL property when $k\gtrsim \varepsilon^{-2}\log^{8}m$ for $N=2$ and $R=1$. Our results for $\fCP{R}$ extend theirs to arbitrary values of $N$ and $R$ and provide tighter bounds even for the case of $N=2$ and $R=1$.

Lastly, Jin et al.~\citep{jin2019faster} extend the fast JLT  for embedding vectors with a Kronecker product structure. They show that the map they propose satisfy the JL property when $k \gtrsim \varepsilon^{-2}\log^{2N-1}m\log (d^N)$~(up to polylog factors) and that projecting a rank one tensor can be done in $\bigo(N d\log d + k)$.
While the upper bound we derive for $\fTT{R}$ is comparable, the choice of the rank parameter gives more flexibility to control the trade-off between accuracy and computational efficiency. In particular, computing $\fTT{R}(\Xt)$ can be considerably faster than the method proposed in~\citep{jin2019faster} when $\Xt$ is a low rank tensor given in the TT format~(see Section~\ref{section:related.work}).

\section{Preliminaries}

In this section, we introduce our notations and present the necessary background on tensor algebra, tensor decomposition and  random projections. More details can be found in~~\citep{kolda2009tensor,vempala2005random, dasgupta2003elementary}.

\subsection{Notations}
We use lower case bold letters for vectors~(\eg   $\ab,\bb$, ...), upper case bold letters for matrices~(\eg $\Ab,\Bb$, ...), and bold calligraphic letters for higher order tensors~(\eg  $\At,\Bt$, ...). 
If $\vb \in \RR^{d_1}$ and $\ub\in\RR^{d_2}$, we use $\vb \otimes \ub \in \RR^{d_1d_2}$ to denote the Kronecker product between vectors. The 2-norm of a vector $\ub$ is denoted by $\norm{\ub}_2$ or simply $\norm{\ub}$. The Khatri-Rao product is defined as the “matching column-wise” Kronecker product: if $\Ab\in\RR^{m\times R}$ and $\Bb\in\RR^{n \times R}$, it is denoted by $\Ab\odot\Bb$ and given by $ [\ab_1\otimes\bb_1\cdots \ab_R\otimes\bb_R]\in\RR^{mn\times R}$. We use the symbol "$\circ$" to denote the outer product~(or tensor product) between vectors. Given a matrix $\Sb \in \RR^{d_1\times d_2}$, we use $\vectorize(\Sb)\in \RR^{d_1.d_2}$ to denote the column vector
obtained by concatenating the columns of $\Sb$. The $d\times d$ identity matrix will be written as $\Ib_d$ and the transpose of a matrix $\Ab$ is denoted by $\Ab^\ts$. For any integer $k$ we use $[k]$ to denote the set of integers from 1 to $k$. For scalars $x,y\in\RR$, we use  $x \gtrsim  y$ to denote that $x\geq c y$ for some constant $c$. %

\subsection{Tensors}
A \textit{$N$-th order tensor} $\St\in\RR^{d_1\times \cdots \times d_N}$ can simply be seen as a multidimensional array $(\St_{i_1,\cdots,i_N}: i_n\in[d_n],n\in[N])$. The inner product between tensors is defined by $\inner\St\Tt = \sum_{i_1,\cdots,i_N}\St_{i_1,\cdots ,i_N}\Tt_{i_1,\cdots,i_N}$ for $\Tt\in\RR^{d_1\times\dots\times d_N}$ and the Frobenius norm is defined by $\nbr\St_F^2 = \inner\St\St$. If $\At\in\RR^{I_1\times\dots\times I_N}$ and $\Bt\in\RR^{J_1\times\dots\times J_N}$, we use $\At\otimes\Bt\in\RR^{I_1J_1\times\dots\times I_NJ_N}$ to denote the Kronecker product of tensors. The \textit{mode-n} fibers of $\St$ are the vectors obtained by fixing all indices except the $n$th one. The \textit{$n$-th mode matricization} of $\St$ is the matrix having the mode-$n$ fibers of $\St$ for columns\footnote{The specific ordering of the fibers does not matter as long as it is consistent across all reshaping operations.} and is denoted by $\St_{(n)}\in\RR^{d_n\times d_1\cdots d_{n-1}d_{n+1}\cdots d_N}$. The vectorization of a tensor is the vector obtained by concatenating its mode-1 fibers, i.e., $\vectorize(\St) = \vectorize(\St_{(1)})$. The notion of matricization can be extended to any subset $I\subset [N]$ of the modes of $\St$, resulting in a matrix $\St_{(I)}$ of size $\prod_{i\in I} d_i \times \prod_{j\in [N]\setminus I} d_j$. 

A rank $R$ \textit{CP decomposition} of a tensor $\St \in\RR^{d_1\times\cdots\times d_N}$ consists in factorizing $\St$  into a sum of $R$ rank one tensors: $\St = \sum_{r=1}^R  \ab^{1}_r \circ \ab^{2}_r  \circ \cdots \ab^{N}_r$ where each $\ab^{n}_r \in \RR^{d_n}$. Stacking the vectors $\ab^{n}_1,\dots,\ab^{n}_R$ into a factor matrix $\Ab^{n}\in\RR^{d_n\times R}$ for each $n\in [N]$, we will concisely denote the CP decomposition by $\St = \CP{\Ab^{1},\cdots,\Ab^{N}}$. 

A rank $R$ \textit{tensor train decomposition} of a tensor $\St\in\RR^{d_1\times\cdots\times d_N}$ consists in factorizing $\St$ into the the product of $N$  $3$rd-order core tensors $\Gt^1\in\RR^{1\times d_1\times R} , \Gt^2\in\RR^{R \times d_2 \times R},\cdots, \Gt^{N-1}\in\RR^{R\times d_{N-1}\times R},\Gt^N\in\RR^{R\times d_N\times 1}$, and is defined\footnote{The general definition of the TT-decomposition allows the rank $R$ to be different for each mode, but this definition is sufficient for the purpose of this paper.} by 
$\St_{i_1,\cdots,i_N} = (\Gt^1)_{i_1,:}( \Gt^2)_{:,i_2,:}\cdots(\Gt^{N-1})_{:,i_{N-1},:}(\Gt^N)_{:,i_N}$, for all indices $i_1\in[d_1],\cdots,i_N\in[d_N]$; we will use the notation $\St=\TT{\Gt^1,\Gt^2,\cdots,\Gt^{N-1},\Gt^N}$ to denote the TT decomposition. 
\subsection{Johnson-Lindenstrauss Transform}
A classical result of Johnson-Lindenstrauss (JL)~\citep{johnson1984extensions} states that any $m$-point set $P$ in $d$ dimension can be linearly projected to $k =\Omega(\varepsilon^{-2}\log{(m)})$ dimensions while approximately preserving the pairwise distances between the points. More precisely, there exists a map $f:\RR^d\to\RR^k (d\gg k)$ such that for all $\ub,\vb \in P$,
\begin{align*}
(1-\varepsilon)\norm{\ub-\vb}^2\leq\norm{f(\ub)-f(\vb)}^2\leq(1+\varepsilon)\norm{\ub-\vb}^2.
\end{align*}
We will call a map satisfying this property a \emph{Johnson-Lindenstrauss transform}~(JLT). One of the simplest examples of a JL transform is the so-called \emph{Gaussian random projection} map $f:\xb\mapsto \frac{1}{\sqrt{k}}\Ab \xb$ where $\Ab\in \RR^{k\times d}$ is a random matrix whose entries are independently drawn from a normal distribution. For a fixed set of input points in $\RR^d$, $f$ will satisfy the JL property with high probability. To cope with the computational cost and storage requirements of Gaussian random projections, sparse and very-sparse random projections were proposed in ~\citep{achlioptas2003database} and~\citep{li2006very} respectively. These maps leverage the fact that  the JL property is preserved even if only a small subset of the entries of $\Ab$ are normal variables while the other ones are set to $0$.

It is easy to see that in order to be a JL transform, a map $f$ must satisfy two fundamental properties: (i) it has to be an expected isometry, i.e. $\Ex{\norm{f(\xb)}^2}=\norm{\xb}^2$, and (ii) the variance of $\norm{f(\xb)}^2$ should quickly decrease to $0$ as the size of the random projection $k$ increases.

\section{Tensorized Random Projections}

As mentioned in the previous section, sparse and very-sparse Gaussian RP reduce time and memory complexity by enforcing the rows of the matrix $\Ab$ in the Gaussian RP $f:\xb\to\frac{1}{\sqrt{k}}\Ab\xb$ to be sparse. In this work, we propose to enforce a low rank tensor structure on the rows of $\Ab$ instead to obtain  better scalability w.r.t. the input dimension, which is crucial when dealing with high-order tensor inputs.  

We present two tensorized random projection maps, $\fTT{R}$ and $\fCP{R}$, relying on the TT and CP decomposition respectively.  These maps embed any tensor $\Xt\in\RR^{d_1\times\cdots\times d_N}$ into $\RR^k$, where $k \ll d_1d_2\cdots d_N$. Considering the case $d_1 = \cdots = d_N = d$ for simplicity, classical random projection maps would require $\bigo(kd^N)$ parameters~(or $\bigo(k\sqrt{d^N})$ with very sparse random projections) which is costly when $N$ is large. In contrast, $\fTT{R}$ and $\fCP{R}$ only require $\bigo(kNdR^2)$ and $\bigo(kNdR)$ parameters respectively. The two maps are constructed similarly: each component of the projection is given by the inner product between the input and a random tensor with a low rank structure~(w.r.t. either the TT or CP decomposition format). Formally, we have the following two definitions:

\begin{definition}\label{TT.definition}
A \emph{TT random projection of  rank $R$} is a linear map $\fTT{R}:\RR^{d_1\times\dots  \times d_N}\to\RR^{k}$ defined component-wise by
$$
\left(\fTT{R}(\Xt)\right)_i :=~\frac{1}{\sqrt{k}}\langle \TT{\Gt_i^1,\Gt_i^2,\cdots,\Gt_i^N},\Xt\rangle,\ i\in[k]
$$
where $\Gt_i^1\in\RR^{1\times d_1\times R} , \Gt_i^2\in\RR^{R \times d_2 \times R},\cdots, \Gt_{i}^{N-1}\in\RR^{R\times d_{N-1}\times R},\Gt_i^N\in\RR^{R\times d_N\times 1}$ for $i\in[k]$, and the entries of each $\Gt_i^n$ for $i\in[k]$, $n\in[N]$ are drawn independently from a Gaussian distribution with mean $0$ and variance $\frac{1}{\sqrt{R}}$ if $n\in \{1,N\} $ and variance $\frac{1}{R}$ if $1<n<N$.
\end{definition}

\begin{definition}\label{CP.definition}
A \emph{CP random projection of rank $R$} is a linear map $\fCP{R}:\RR^{d_1\times\dots  \times d_N}\to\RR^{k}$ defined component-wise by
$$
\left(\fCP{R}(\Xt)\right)_i :=~\frac{1}{\sqrt{k}}\langle \CP{\Ab_i^1,\Ab_i^2,\cdots,\Ab_i^N},\Xt\rangle,\ \  i\in[k]
$$
where each $\Ab_i^n\in\RR^{d_n\times R}$ for $i\in[k]$, $n\in[N]$ and the entries of each $\Ab_i^n$  are drawn independently from a Gaussian distribution with mean $0$ and variance $\left(\frac{1}{R}\right)^\frac{1}{N}$.
\end{definition}

 One can check that applying these projection maps on an input tensor given in the CP or the TT format can be done efficiently: the complexity of computing $\fTT{R}(\Xt)$ is in $\bigo(kNd\max(R,\tilde{R})^3)$ if $\Xt$ is given as a rank $\tilde{R}$ CP or TT tensor, and the complexity for $\fCP{R}(\Xt)$ is in $\bigo(kNd\max(R,\tilde{R})^2)$ if $\Xt$ is in the CP format and in $\bigo(kNd\max(R,\tilde{R})^3)$ if $\Xt$ is in the TT format~(where we assumed $d_1=\dots=d_N=d$ for simplicity).

 Before studying the properties of these tensorized random projections in the next section, we  show how $\fCP{\cdot}$ is equivalent to the tensor random projection map proposed in~ \citep{sun2018tensor}. In this work, the authors introduce the map $$f_\text{TRP}(\Xt) := \frac{1}{\sqrt{k}}(\Ab^1\odot\Ab^2\odot\cdots\odot\Ab^N)^\ts \vectorize(\Xt)\in\RR^k,$$ where each $\Ab^n\in\RR^{d_n\times k}$ for $n \in [N]$ is a random matrix whose entries are i.i.d random variables with mean zero and variance one. One can check that $f_\text{TRP}$ is strictly equivalent to $\fCP{1}$ using basic properties of the CP decomposition. Furthermore, the authors introduce a variance reduction technique with the map $f_{\text{TRP}(T)}$, a scaled average of $T$ independent TRPs, defined by $f_{\text{TRP}(T)}(\Xt) :=\frac{1}{\sqrt{T}}\sum_{t = 1}^T f_\text{TRP}^{(t)}(\Xt)$. Again, one can easily check the strict equivalence between $\fCP{R}$ and $f_{\text{TRP}(T)}$ when $R=T$.

\section{Main Results}
In this section, we present our main results showing that the tensorized projection $\fTT{R}$ and $\fCP{R}$ still benefits from the fundamental properties of Gaussian random projections: they are expected isometry and the variance of the norm of the projections decreases to $0$ as the embedding dimension $k$ grows. These results imply that, in addition to be particularly efficient in terms of storage requirement and computational cost, these maps are JL transforms: they approximately preserve Euclidean distances between projected points. Moreover, our analysis will show that there is a crucial difference between the two tensorized random projections: as the order of the input tensor $\Xt$ grows, the embedding dimension of $\fCP{R}$ needs to grow exponentially in comparison to the one of $\fTT{R}$ in order to achieve the same distortion ratio $\varepsilon$.
Our results rely on the following theorem which shows that both maps are expected isometries and gives bounds on the variance of the two projections. 

\begin{theorem}\label{variance}
Let $\Xt\in\RR^{d_1\times d_2 \times \cdots \times d_N}$. The random projection maps $\fTT{R}$ and $\fCP{R}$~(see Definitions~\ref{TT.definition} and~\ref{CP.definition}) satisfy the following properties:
\begin{itemize}[leftmargin=*]
    \item $\Ex{\norm{\fCP{R}(\Xt)}_2^2} = \Ex{\norm{\fTT{R}(\Xt)}_2^2} = \norm{\Xt}_F^2$
    \item $\Var{\norm{\fTT{R}(\Xt)}_2^2}\leq\frac{1}{k}(3\left(1+\frac{2}{R}\right)^{N-1}-1)\left\|{\Xt}\right\|_F^4$
    \item $\Var{\norm{\fCP{R}(\Xt)}^2_2}\leq~\frac{1}{k}\left(3^{N-1}\left(1+\frac{2}{R}\right)-1\right)\nbr{\Xt}_F^4$
\end{itemize}
\end{theorem}
The proof of this theorem for the TT random projection map is given in the next section and the proof for $\fCP{R}$ can be found in the Appendix.

In the case of vector inputs, \ie $N=1$, we recover the classical expression for the variance of Gaussian random projections given by $\Var{\norm{f(\xb)}}^2=\frac{2}{k}\norm{\xb}^4$~(note that in this setting $R$ is necessarily equal to $1$ since $N=1$).

It is worth mentioning  that the only inequality used to derive the bounds comes from the sub-multiplicativity of the Frobenius norm applied to matricizations of the input tensor $\Xt$. For example, for the case of order 2 input tensors, \ie matrices, the variance of $\fTT{R}$ is given by 

$\Var{\norm{\fTT{R}(\Xb)}^2} = \frac{1}{k}\left(2\norm{\Xb}^4_F+\frac{6}{R} \Tr{(\Xb^\ts\Xb)^2}\right).$

Comparing now the bounds on the variance of $\fTT{R}$ and $\fCP{R}$, we observe that while both bounds have an exponential dependency on the order $N$ of the input tensors, slightly increasing the rank $R$ of the TT random projection mitigates this dependency while it has no effect for the CP random projection. This shows that $\fCP{R}$ is not a suitable RP since $k$ has to grow exponentially in $N$ in order to approach the variance of classical Gaussian random projections.
Using the bounds on the variance of the projections, we can now derive lower bounds on the size $k$ of the random projections $\fCP{R}$ and $\fTT{R}$ needed to satisfy the JL property with high probability.

\begin{theorem}\label{JLpropertyCPTT}
Let $P \subset\RR^{d_1\times d_2 \times \cdots \times d_N} $ be a set of $m$ order $N$ tensors. Then, for any $\varepsilon>0$ and any $\delta>0$, the following hold simultaneously for all $\Xt\in P$:
\begin{itemize}
    \item if $k\gtrsim \varepsilon^{-2}(1+2/R)^N\mathrm{log}^{2N}\left(\frac{m}{\delta}\right)$ then\\
    $\PP(\left\|{\fTT{R}(\Xt)}\right\|_2^2 = (1\pm\varepsilon)\left\|{\Xt}\right\|_F^2)\geq 1-\delta$,
    \item if $k\gtrsim \varepsilon^{-2}3^{N-1}(1+2/R)\mathrm{log}^{2N}\left(\frac{m}{\delta}\right)$ then\\
    $\PP(\left\|{\fCP{R}(\Xt)}\right\|_2^2 = (1\pm\varepsilon)\left\|{\Xt}\right\|_F^2)\geq 1-\delta$.
\end{itemize}
\end{theorem}

\subsection{Comparison to related work}\label{section:related.work} We conclude this section by comparing the previous theorem with the closest related work. 
Jin et al.~\citep{jin2019faster} proposed a Kronecker structured JL transform satisfying the JL property for $m$ points with probability $1-\delta$ as soon as $k \gtrsim \varepsilon^{-2}\log^{2N-1}\left(\frac{m}{\delta}\right)\log (d^N)$, up to polylog factors. Our results are similar to theirs but differ in one key aspect. In their work, projecting a rank one tensor can be done in $\bigo(N d\log d + k)$. Hence, by linearity, projecting a tensor of rank $\tilde{R}$ given in the CP format can be done in $\bigo(\tilde{R}(N d\log d + k))$. However, low rank tensors given in the TT format cannot be efficiently projected using their method\footnote{Indeed, almost all low rank TT tensors have exponentially large CP rank~(see e.g, Theorem~1 in~\citep{khrulkov2018expressive})}, In contrast, $\fTT{R}(\Xt)$ and $\fCP{R}(\Xt)$ can both be computed in $\bigo(k N d \max(R,\tilde{R})^3)$ when $\Xt$ is given as a rank $\tilde{R}$ CP or TT tensor; our approach is thus better suited for inputs given in the TT format. In~\citep{sun2018tensor}, they proposed a tensor random projection map for sub-Gaussian random variables.
They give a lower bound of $k\gtrsim \varepsilon^{-2}\log^8\left(\frac{m}{\delta}\right)$ only for the case of order $2$ input tensors, treating the rank parameter $R$ as a constant. Moreover, even in the case of order $2$ input tensors our lower bound of $\varepsilon^{-2}(1+2/R)\log^4\left(\frac{m}{\delta}\right)$ is tighter than the one they provide.
\section{Proofs}
In this section, we present the proofs of our results for the random projection map $\fTT{R}$. The techniques used for the map $\fCP{R}$ are of a similar flavor and can be found in the Appendix. 

 \subsection{Proof of  Theorem~\ref{variance}: TT case}
\paragraph{Expected isometry.} We start by showing that $\fTT{R}$ is an expected isometry, \ie that $\EE\left\|{\fTT{R}(\Xt)}\right\|_2^2    =\left\|{\Xt}\right\|_F^2$. 
Let $y_i =\langle\TT{\Gt_i^1,\Gt_i^2,\cdots,\Gt_i^N},\Xt\rangle$ and $\yb = [y_1,y_2,\cdots,y_k]$. With these definitions we have $\fTT{R}(\Xt) = \frac{1}{\sqrt{k}} \yb$ and it is thus sufficient to find $\EE[y_1^2]$. 
To lighten the notation, let $\Gt^n=\Gt^n_1$ for each $n\in [N]$ and let $\St = \TT{\Gt^1,\Gt^2,\cdots,\Gt^N}$. We have
\begin{align*}
\EE[y_1^2] 
&=
\EE[\langle \St,\Xt\rangle^2] =
\EE[\langle \St\otimes \St,\Xt\otimes \Xt\rangle]\\
&=
\langle \EE[\St\otimes \St],\Xt\otimes \Xt\rangle.
\end{align*}
Using the fact that the core tensors $\Gt^n$ are independent, we have
\begin{align*}
\MoveEqLeft
\EE[\St\otimes \St]
=
\EE[\TT{\Gt^1\otimes\Gt^1,\cdots,\Gt^N\otimes\Gt^N }] \\
&=
\TT{\EE[\Gt^1\otimes\Gt^1],\cdots,\EE[\Gt^N\otimes\Gt^N] }.
\end{align*}
Now,  for $1<n<N$, since the entries of each core tensor $\Gt^n$ are i.i.d. Gaussian variables with mean $0$ and variance $1/R$, we have
$$ \EE[\Gt^n\otimes\Gt^n]=\frac{1}{R} \vectorize(\Ib_R)\circ \vectorize(\Ib_{d_n})\circ\vectorize(\Ib_R).$$
Similarly,  
$\EE[\Gt^1\otimes\Gt^1]=\frac{1}{\sqrt{R}}\vectorize(\Ib_{d_1})\circ\vectorize(\Ib_R)$
and 
$\EE[\Gt^N\otimes\Gt^N]=\frac{1}{\sqrt{R}}\vectorize(\Ib_R)\circ\vectorize(\Ib_{d_N}).$

A careful but straightforward derivation consequently shows that
$
\EE[ \St\otimes \St] 
=
 \vectorize(\Ib_{d_1}) \circ \cdots\circ \vectorize(\Ib_{d_N}),
$
which implies 
$
\EE[y_1^2] =
\langle \EE[\St\otimes \St],\Xt\otimes \Xt\rangle =
\norm{\Xt}^2_F.
$
From which  $\EE\left\|{\fTT{R}(\Xt)}\right\|_2^2    =\left\|{\Xt}\right\|_F^2$  directly follows.

\paragraph{Bound on the variance of $\fTT{R}$.}
In order to bound the variance of $\norm{\yb}_2^2$ we need to bound $\EE[\left\|{\yb}\right\|_2^4]$. We have
\begin{align*}
\EE[\left\|{\yb}\right\|_2^4]=\sum_{i=1}^k\EE[ y_i^4]+\sum_{i\neq j}\EE [y_i^2y_j^2].
\end{align*} 
Since $y_i$ and $y_j$ are independent whenever $i\neq j$ and $\EE[y_i^2]=\nbr{\Xt}_F^4$ for all $i$, the second summand is equal to $k(k-1)\nbr{\Xt}_F^4$. We now derive a bound on $\EE[y_1^4]$.

Our proof relies on the following technical lemmas%
. The first one is a direct consequence of \textit{Isserlis' theorem}~\citep{isserlis1918formula} and the second one follows from standard properties of the Wishart distribution~(see \eg Section 3.3.6 of~\citep{gupta2018matrix}).

\begin{lemma}\label{Innerproduct}
Let $\Ab\in\RR^{m\times n}$ be a random matrix whose entries are i.i.d normal random variables with mean zero and variance $\sigma^2$, and let $\Bb\in\RR^{m\times n}$ be a (random) matrix independent of $\Ab$. Then,
\begin{align*}
\EE\langle\Ab,\Bb\rangle^4=3\sigma^4\EE\left\|{\Bb}\right\|_F^4.
\end{align*}
\end{lemma}
\begin{proof}
Setting $\ab = \vectorize(\Ab)\in\RR^{mn}$ and $\bb = \vectorize(\Bb)\in\RR^{mn}$, we have
\begin{align*}
\EE\langle\Ab,\Bb\rangle^4 &=\EE\langle\ab,\bb\rangle^4\\
&= \EE\langle\ab^{\otimes 4},\bb^{\otimes 4}\rangle=\langle\EE[\ab^{\otimes 4}],\EE[\bb^{\otimes 4}]\rangle,
\end{align*}
where the last equality is obtained by using the independence between $\ab$ and $\bb$. Element-wise, by using Isserlis' theorem~\citep{isserlis1918formula} and using the fact that $\ab\sim\mathcal{N}(\zero,\sigma^2\Ib)$ we have, 
\begin{align*}
\MoveEqLeft
(\EE[\ab^{\otimes 4}]))_{i_1,i_2,i_3,i_4} 
=
\EE[\ab_{i_1}\ab_{i_2}\ab_{i_3}\ab_{i_4}] \\
&= 
\EE[\ab_{i_1}\ab_{i_2}]\EE[\ab_{i_3}\ab_{i_4}] + \EE[\ab_{i_1}\ab_{i_3}]\EE[\ab_{i_2}\ab_{i_4}] \\ 
&+ \EE[\ab_{i_1}\ab_{i_4}]\EE[\ab_{i_2}\ab_{i_3}]\\
&= (\delta_{i_1i_2}\delta_{i_3i_4} + \delta_{i_1i_3}\delta_{i_2i_4} + \delta_{i_1i_4}\delta_{i_2i_3})\sigma^4,
\end{align*} 
where $\delta$ is the Kronecker symbol. Therefore, letting $\Delta_{i_1i_2i_3i_4}=\delta_{i_1i_2}\delta_{i_3i_4} + \delta_{i_1i_3}\delta_{i_2i_4} + \delta_{i_1i_4}\delta_{i_2i_3}$, we obtain
\begin{align*}
\MoveEqLeft
\EE\langle\Ab,\Bb\rangle^4 
=
\sum_{i_1,i_2,i_3,i_4} \EE[\ab_{i_1}\ab_{i_2}\ab_{i_3}\ab_{i_4}]  \EE[\bb_{i_1}\bb_{i_2}\bb_{i_3}\bb_{i_4}] \nonumber\\
&=
\sigma^4\sum_{i_1,i_2,i_3,i_4}\Delta_{i_1i_2i_3i_4}\EE[\bb_{i_1}\bb_{i_2}\bb_{i_3}\bb_{i_4}]\nonumber\\
&=
\sigma^4~\EE\left[\sum_{i_1,i_3}\bb_{i_1}^2\bb_{i_3}^2 +\sum_{i_1,i_4}\bb_{i_1}^2\bb_{i_4}^2 +\sum_{i_1,i_2}\bb_{i_1}^2\bb_{i_2}^2  \right]\nonumber\\
&=
3\sigma^4\EE\nbr{\Bb}_F^4.&\qedhere
\end{align*}
\end{proof}

\begin{lemma}\label{Wishartidentity}
Let $\Ab\in\RR^{m\times n}$ be a random matrix whose entries are i.i.d Gaussian random variables with mean zero and variance $\sigma^2$, and let $\Bb\in\RR^{p\times m}$ be a (random) matrix independent of $\Ab$. Then,
\begin{align*}
\EE\nbr{\Bb\Ab}_F^4&=n\sigma^4\left(n~\EE\left\|{\Bb}\right\|_F^4+2\EE\tr((\Bb^\ts\Bb)^2)\right)\\
&\leq\sigma^4n(n+2)\EE\nbr{\Bb}_F^4.
\end{align*}
\end{lemma}
\begin{proof}
By definition of the Frobenius norm we have 
\begin{align*}
\EE\nbr{\Bb\Ab}_F^4 = \EE\left[\tr\left(\Bb^\ts\Bb\Ab\Ab^\ts\right)\tr\left(\Bb^\ts\Bb\Ab\Ab^\ts\right)\right].
\end{align*}
Since $\Ab_{ij}\sim\mathcal{N}(0,\sigma^2)$ for any $i\in[m],j\in[n]$, $\Ab\Ab^\ts\in\RR^{m\times m}$ is a random symmetric positive definite matrix following a Wishart distribution with parameters $m,n$ and $\sigma^2\Ib_m \in \RR^{m\times m}$. Therefore,
\begin{align*}
\MoveEqLeft
\EE\left[\tr\left(\Bb^\ts\Bb\Ab\Ab^\ts\right)\tr\left(\Bb^\ts\Bb\Ab\Ab^\ts\right)\right]\\
&=
n\sigma^4\left(n\EE\left\|{\Bb}\right\|_F^4+2\EE\tr((\Bb^\ts\Bb)^2)\right)\nonumber\\
&\leq
\sigma^4n(n+2)\EE\left\|{\Bb}\right\|_F^4,\nonumber
\end{align*}
where the equality follows from standard properties of the Wishart distribution~(see \textit{e.g.}, Section 3.3.6 of~\citep{gupta2018matrix}), and the inequality follows from the sub-multiplicativity of the Frobenius norm. \qedhere

\end{proof}
Let us now start by defining the tensor $\Mt^n\in\RR^{R\times d_1\times d_2\times \cdots \times d_{n-1}}$ for each $2\leq n \leq N$ component-wise by

\begin{align*}
    &\Mt^n_{r,i_1,\dots,i_{n-1}} = \sum_{i_n,\dots,i_N\atop r_n,\dots,r_{N-1}} (\Gt^n)_{r,i_n,r_n}
(\Gt^{n+1})_{r_n,i_{n+1},r_{n+1}}\\ &\dots (\Gt^{N-1})_{r_{N-2},i_{N-1},r_{N-1}}(\Gt^N)_{r_{N-1},i_{N}} \Xt_{i_1,\dots,i_N},
\end{align*}

for each $r\in [R]$, $i_1\in[d_1],\dots,i_{n-1}\in [d_{n-1}]$.
In some sense, $\Mt^n$ is the tensor obtained by removing the first $n-1$ cores from the computation of $y_1=\langle\TT{\Gt^1,\Gt^2,\cdots,\Gt^N},\Xt\rangle$. With this definition, one can check that
$\noindent{\tiny}\bullet\langle\TT{\Gt^1,\Gt^2,\cdots,\Gt^N}\Xt\rangle
= \langle\Gt^1,\Mt^2\rangle,$
$\noindent{\tiny}\bullet\Mt^N_{(1)}=(\Gt^N)_{(1)}\Xt_{(N)}$ and
${\tiny}\bullet\Mt_{(1)}^n = (\Gt^n)_{(1)}(\Mt^{n+1})_{(1,n+1)}$ for each $n\in[N],$ where $(\Mt^{n+1})_{(1,n+1)}\in\mathbb{R}^{Rd_n\times d_1\dots d_{n-1}}$ denotes the matricization of $\Mt^{n+1}$ obtained by mapping its first  and last modes to rows and the other ones to columns. Let $\sigma^2_n$ denote the variance used to draw the entries of each core $\Gt^n$.
Using Lemma~\ref{Innerproduct} we obtain
\begin{align*}
\EE y_1^4 &= \EE\langle\TT{\Gt^1,\Gt^2,\cdots,\Gt^N},\Xt\rangle^4=\EE\langle\Gt^1,\Mt_{(1)}^2\rangle^4\\
&=3\sigma^4_1\EE\left\|{\Mt^2_{(1)}}\right\|_F^4=3\sigma^4_1\EE\left\|{(\Gt^2)_{(1)}\Mt^3_{(1,3)}}\right\|_F^4.
\end{align*}
 Using the fact that the Frobenius norm of a tensor is constant across all matricizations and by Lemma~\ref{Wishartidentity} we get
\begin{align*}
\EE[&y_1^4]
=
3\sigma^4_1\EE\left\|{(\Gt^2)_{(1)}\Mt^3_{(1,3)}}\right\|_F^4\\
&\leq
3\sigma^4_1\sigma^4_2R(R+2)\EE\left\|{\Mt^3_{(1,3)}}\right\|_F^4 \\
&=
3\sigma^4_1\sigma^4_2R(R+2)\EE\left\|{\Mt^3_{(1)}}\right\|_F^4 \\
&=
3\sigma^4_1\sigma^4_2R(R+2)\EE\left\|{(\Gt^3)_{(1)}\Mt_{(1,4)}^4}\right\|_F^4\\
&\leq
3\sigma^4_1\sigma^4_2\sigma^4_3R^2(R+2)^2\EE\left\|{\Mt^4_{(1,4)}}\right\|_F^4
\end{align*}

Similarly, using successive applications of Lemma~\ref{Wishartidentity} it then follows that
\begin{align*}
&\EE[y_1^4]
\leq
3\sigma^4_{1}\cdots\sigma^4_{{N-1}}R^{N-2}(R+2)^{N-2}\EE\left\|\Mt^N_{(1)}\right\|_F^4\\
&=
3\sigma^4_{1}\cdots\sigma^4_{{N-1}}R^{N-2}(R+2)^{N-2}\EE\left\|(\Gt^N)_{(1)}\Xt_{(N)}\right\|_F^4\\
&\leq
3\sigma^4_{1}\cdots\sigma^4_{{N}}R^{N-1}(R+2)^{N-1}\left\|\Xt_{(N)}\right\|_F^4\\
&=
3\frac{1}{R}\left(\frac{1}{R^2}\right)^{N-2}\frac{1}{R}R^{N-1}(R+2)^{N-1}\left\|{\Xt}\right\|_F^4\\
&=
3\left(1+\frac{2}{R}\right)^{N-1}\left\|{\Xt}\right\|_F^4.
\end{align*}

Therefore we obtain $
\EE\left\|{\yb}\right\|_2^4
\leq 3k\left(1+\frac{2}{R}\right)^{N-1}\left\|{\Xt}\right\|_F^4+k(k-1)\left\|{\Xt}\right\|_F^4.
$
Finally,
\begin{align*}
\MoveEqLeft
\Var{\nbr{\fTT{R}(\Xt)}_2^2}
= \EE[\|k^{-\frac{1}{2}}\yb\|_2^4]-\EE[ \|k^{-\frac{1}{2}}\yb\|_2^2]^2 \\
&=\frac{1}{k^2}\EE\nbr{\yb}_2^4-\nbr{\Xt}_F^4\\
&
\leq~\frac{1}{k}\left[3\left(1+\frac{2}{R}\right)^{N-1}-1\right]\nbr{\Xt}_F^4.
\end{align*}

\subsection{Proof of Theorem~\ref{JLpropertyCPTT}: TT case}
Theorem~\ref{JLpropertyCPTT} for the map $\fTT{R}$ directly follows from the following concentration bound.
\begin{theorem}\label{concenterationboundTT}
Let $\Xt\in\RR^{d_1\times d_2 \times \cdots \times d_N}$. There exist absolute constants $C$ and $K>0$ such that the  random projection map $\fTT{R}$~(see  Definition~\ref{TT.definition}) satisfies
\begin{align*}
\PP\left(\left|\left\|{\fTT{R}(\Xt)}\right\|_2^2 - \left\|{\Xt}\right\|_F^2\right| \geq \varepsilon\left\|{\Xt}\right\|_F^2\right)\leq\\
C\exp\left[-~\frac{(\sqrt{k}\varepsilon)^\frac{1}{N}}{(3K)^\frac{1}{2N}\sqrt{1+2/R}}\right].
\end{align*}
\end{theorem}

To show this concentration bound, we will use the following extension of the Hanson-Wright inequality whose proof can be found in~\citep{schudy2012concentration}.
\begin{theorem}\textbf{(Hypercontractivity Concentration Inequality)}\label{thm:hypercontractivity}
Consider a degree $q$ polynomial $f(Y)= f(Y_1,\dots,Y_n)$ of independent centered Gaussian or Rademacher random variables $Y_1,\dots,Y_n$. Then for any $\lambda>0$
\begin{align*}
\PP[\left|f(Y)-\EE\left[f(Y)\right]\right|\geq \lambda]\leq e^2.e^{-\left(\frac{\lambda^2}{K.\Var{\left[f(Y)\right]}}\right)^\frac{1}{q}},
\end{align*}
where $\Var{\left[f(Y)\right]}$ is the variance of the random variable $f(Y)$ and $K>0$ is an absolute constant. 
\end{theorem}
Using the bound on the variance of $\left\|{\fTT{R}(\Xt)}\right\|_2^2$ and the fact that $\left\|{\fTT{R}(\Xt)}\right\|_2^2$ is a polynomial of degree $2N$ of independent Gaussian random variables~(the entries of the core tensors $\Gt_i^1,\Gt_i^2,\cdots,\Gt_i^N$), we can use Theorem~\ref{thm:hypercontractivity} to obtain
\begin{align*}
&\PP\left[\left|\left\|{\fTT{R}(\Xt)}\right\|_2^2-\left\|{\Xt}\right\|_F^2\right|\geq \lambda\right]\\
&\leq e^2\exp\left[-\left(\frac{\lambda^2}{K \Var{\left\|{\fTT{R}(\Xt)}\right\|_2^2}}\right)^\frac{1}{2N}\right].
\end{align*}
Let $C = e^2$ and let $\lambda = \varepsilon\left\|{\Xt}\right\|_F^2$ , we finally get
\begin{align*}
&\PP\left[\left|\left\|{\fTT{R}(\Xt)}\right\|_2^2-\left\|{\Xt}\right\|_F^2\right|\geq  \varepsilon\left\|{\Xt}\right\|_F^2\right]\\
&\leq 
C\exp\left[-~\left(\frac{ k\varepsilon^2\left\|{\Xt}\right\|_F^4}{3K(1+2/R)^{N-1}\left\|{\Xt}\right\|_F^4}\right)^\frac{1}{2N}\right]\\
&\leq
C\exp\left[-~\frac{(\sqrt{k}\varepsilon)^\frac{1}{N}}{(3K)^\frac{1}{2N}\sqrt{1+2/R}}\right],
\end{align*}
where the last inequality follows from the fact that
$$(1+2/R)^\frac{N-1}{2N} \leq \sqrt{1+2/R}.$$

\section{Experiments}

\begin{figure*}[th!]
\begin{center}
\includegraphics[width=\textwidth]{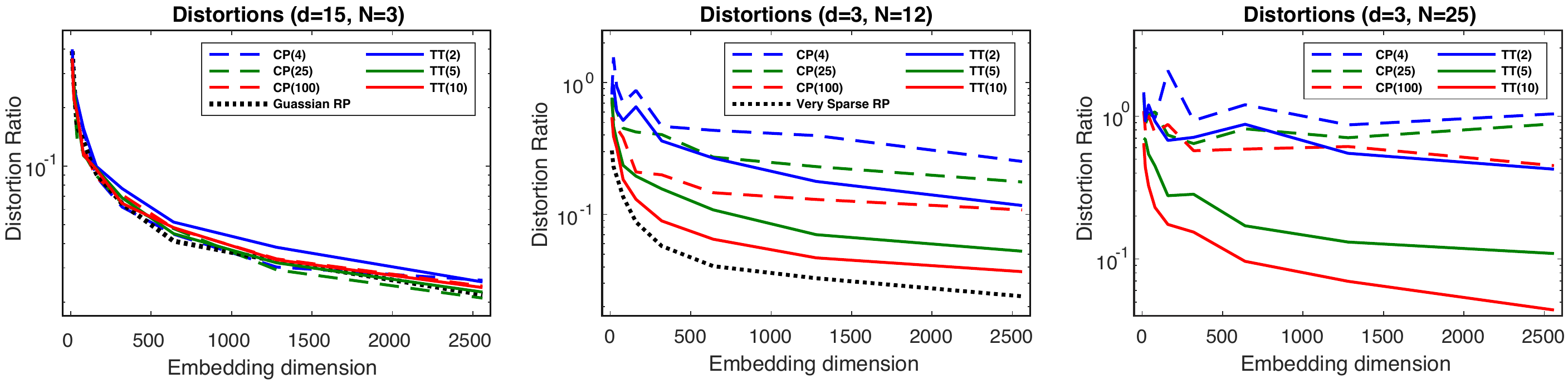}
\caption{Comparison of the distortion ratio of $\fTT{R}$,$\fCP{R}$, and Gaussian RP for different value of the rank parameter $R$ for small-order~(left), medium-order~(center) and high-order~(right) input tensors.}

\label{figure:distortions}
\end{center}
\end{figure*}

\begin{figure}[ht!]
\begin{center}
\includegraphics[width=0.39\textwidth]{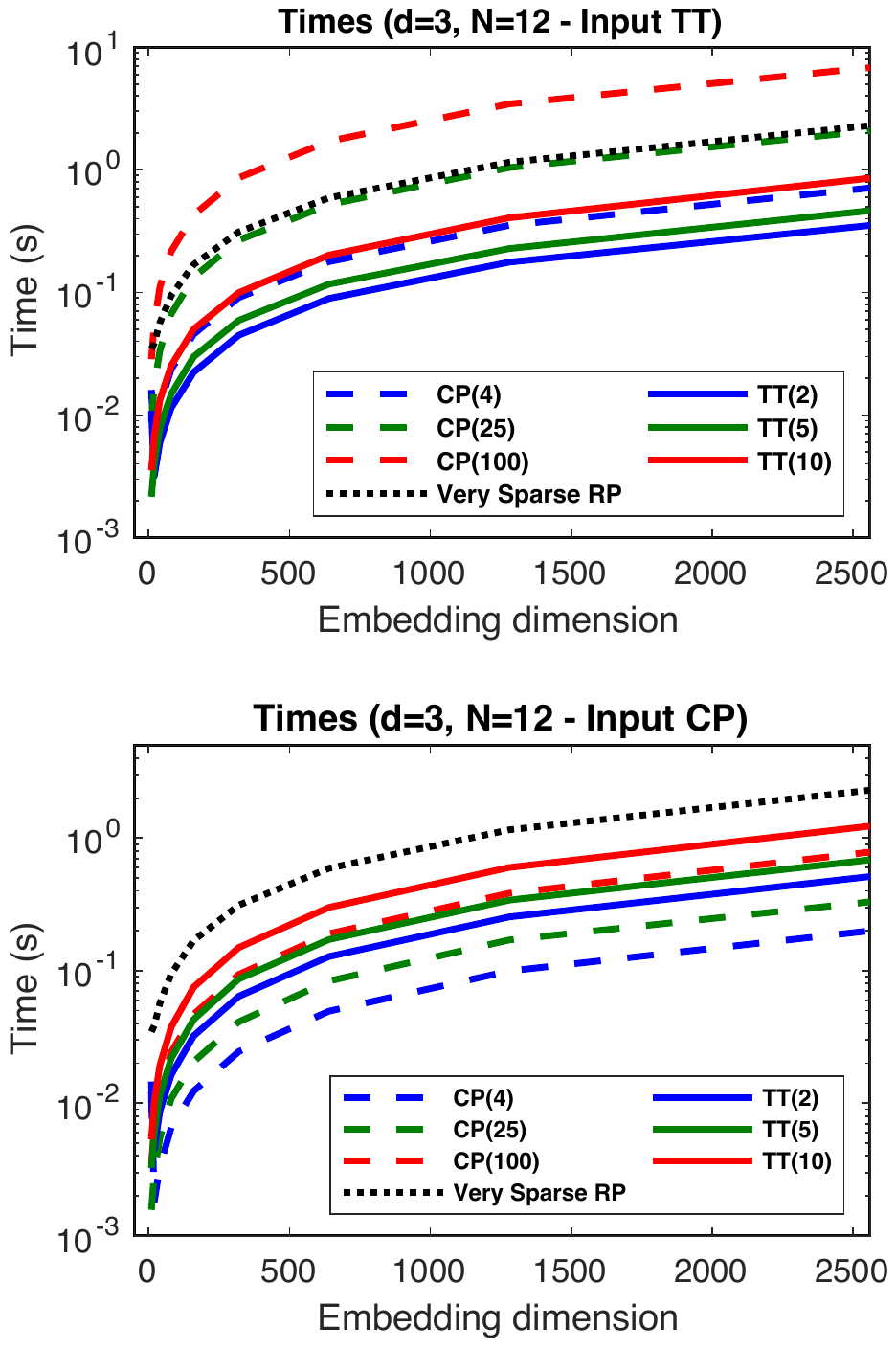}
\caption{Comparison of embedding time between tensorized and very sparse RP for the medium-order case~($d = 3,N = 12$) when the input is given in the TT format~(top) or CP format~(bottom).\vspace*{-0.5 cm}} 
\label{figure:times}
\end{center}
\end{figure}
In this section we compare the embedding quality of the tensorized projection maps $\fTT{R}$, $\fCP{R}$ and Gaussian RP in a simulation study\footnote{For these experiments we use \textit{Tensor Toolbox v3.1}~\citep{TTB_Software} and \textit{TT-Toolbox v2.2}~\citep{TT_Toolbox}.}. In particular, we investigate the effect of the rank parameter $R$ for different sizes and orders of input tensors.
We first randomly generate an $N$-th order $d$-dimensional tensor $\Xt$~(\ie vector of size $d^N$) with unit norm in the TT format with rank $\tilde{R} = 10$. To assess how well the tensorized maps scale to very high order tensors, we consider three cases: $\noindent{\tiny}\bullet~$small-order: $(d=15,N=3)$, $\noindent{\tiny}\bullet~$medium-order: $(d=3,N=12)$ and $\noindent{\tiny}\bullet~$high-order $(d=3,N=25)$.

We compare several values of the rank parameter for the two tensorized map: $R = 4,25,100$ for $\fCP{R}$ and $R = 2,5,10$ for $\fTT{R}$. Note that these values correspond to roughly the same number of parameters for the two maps since $\fTT{R}$ requires the storage of $(N-2)dR^2+2dR$ parameters while $\fCP{R}$ only needs $NdR$. Additional experiment on image data from the CIFAR-10 dataset~\citep{krizhevsky2009learning} are presented in Appendix~\ref{app:xp-cifar}.

The quality of embedding is evaluated using the distortion ratio metric defined by  $D(f,\Xt) = \abs{\frac{\nbr{f(\Xt)}^2}{\nbr{\Xt}^2}-1}$. Due to memory limitation, we compare  tensorized RP with Gaussian RP for the small-order case tensors and with very sparse RP~\citep{li2006very} for medium-order tensors~(the high-order case cannot be handled with Gaussian or very sparse RP).

The average distortion ratios over 100 trials are reported as a function of the embedding dimension $k$ in Figure~\ref{figure:distortions}. In the small-order case, we see that $\fTT{R}$ and $\fCP{R}$ perform similarly to Gaussian RP for all values of the rank parameter. In the medium-order case, we see that the rank of the tensorized RP significantly affects the quality of the embedding. Moreover, $\fCP{R}$ struggles to achieve a good distortion ratio even when $R=100$ while $\fTT{R}$ almost reaches the performance of very sparse RP. This behavior is accentuated in the high-order case where $\fCP{R}$ obtains poor performances even for high values of $R$ and $k$ while $\fTT{R}$ provides good embeddings for $R=5,10$. Note that this behavior is expected from our theoretical analysis.%

To illustrate the time complexity of the algorithms, we report the average running time needed to project the input tensor for the medium-order case in Figure~\ref{figure:times}, when $\Xt$ is either given as a TT or a CP tensor of rank 10. We see that $\fTT{R}$~(resp. $\fCP{R}$) is more efficient when the input tensor is given in the TT format~(resp. CP format), which is somehow expected. We also report the average running time needed to project the different input tensor in medium-order case $(d = 3, N = 8, 11, 12, 13)$ with respect to the dimension $d^N$ (Appendix~\ref{app:xp-time}). It is also worth observing that $\fTT{R}$ is always faster than very sparse RP  while it is not the case for $\fCP{R}$.

\section{Conclusion}

We propose a novel efficient RP technique for high-order tensor data: tensorized random projections maps. We theoretically and empirically studied two tensorized maps relying on the CP and TT deocmposition format, respectively. Our theoretical analysis and simulation study show that the TT format is better suited than the CP format for tensorizing random projections. 

Future work include leveraging and extending our theoretical results to design efficient sketching algorithms for high-order tensor data. In particular, we plan to develop fast low rank approximation algorithms  for matrices given in the TT format, which could prove particularly useful for designing efficient PCA and CCA algorithms for high-dimensional tensor data.
\subsection*{Acknowledgment}
This research is supported by the Canadian Institute
for Advanced Research (CIFAR AI chair program). This work was completed while Beheshteh T. Rakhshan interned at Montreal Institute for Learning Algorithms (Mila), Montreal, QC. 
{
\bibliography{main.bib}
}

\clearpage
\appendix
\newpage
\onecolumn
\section*{\centering \LARGE Tensorized Random Projections \\ \vspace*{0.3cm}(Supplementary Material)}

\section{Proof of the Theorems for the CP case }\label{app:proof}
\subsection{Proof of Theorem~\ref{variance}: CP case}
%%%%%%%%%%lemma.CP.Expectedisometry%%%%%%%%%%
\begin{theorem*}
Let $\Xt\in\RR^{d_1\times d_2 \times \cdots \times d_N}$. The random projection maps $\fTT{R}$ and $\fCP{R}$~(see Definitions~\ref{TT.definition} and~\ref{CP.definition}) satisfy the following properties:

$\bullet~\Ex{\norm{\fCP{R}(\Xt)}_2^2} = \Ex{\norm{\fTT{R}(\Xt)}_2^2} = \norm{\Xt}_F^2,$

$\bullet~\Var{\norm{\fTT{R}(\Xt)}_2^2}\leq\frac{1}{k}(3\left(1+\frac{2}{R}\right)^{N-1}-1)\left\|{\Xt}\right\|_F^4,$

$\bullet~\Var{\norm{\fCP{R}(\Xt)}^2_2}\leq~\frac{1}{k}\left(3^{N-1}\left(1+\frac{2}{R}\right)-1\right)\nbr{\Xt}_F^4.$
\begin{proof}

\textbf{Expected isometry.} We start by showing that $\fCP{R}$ is an expected isometry, \ie that $\EE\left\|{\fCP{R}(\Xt)}\right\|_2^2 =\left\|{\Xt}\right\|_F^2$. Let $y_i =\langle\CP{\Ab_i^1,\Ab_i^2,\cdots,\Ab_i^N},\Xt\rangle$ and $\yb = [y_1,y_2,\cdots,y_k]$. With these definitions we have $\fCP{R}(\Xt) = \frac{1}{\sqrt{k}} \yb$ and it is thus sufficient to find $\EE[y_1^2]$. 
To lighten the notation, let $\Ab^n=\Ab^n_1$ for each $n\in [N]$ and let $\Tt = \CP{\Ab^1,\Ab^2,\cdots,\Ab^N}$. We have
\begin{align*}
\EE[y_1^2] 
&=
\EE[\langle \Tt,\Xt\rangle^2] 
=
\EE[\langle \Tt\otimes \Tt,\Xt\otimes \Xt\rangle]\\
&=
\langle \EE[\Tt\otimes \Tt],\Xt\otimes \Xt\rangle.
\end{align*}
Using the fact that the factor matrices $\Ab^n$ are independent, we have
\begin{align*}
\EE[\Tt\otimes \Tt]
&=
\EE[\CP{\Ab^1\otimes\Ab^1,\cdots,\Ab^N\otimes\Ab^N }] \\
&=
\CP{\EE[\Ab^1\otimes\Ab^1],\cdots,\EE[\Ab^N\otimes\Ab^N] }.
\end{align*}
Now,  for $n\in[N]$, since the entries of each factor matrix $\Ab^n$ are i.i.d. Gaussian random variables with mean $0$ and variance $(\frac{1}{R})^\frac{1}{N}$, we have
$$ \EE[\Ab^n\otimes\Ab^n]= \left(\frac{1}{R}\right)^\frac{1}{N} \vectorize(\Ib_{d_n})\circ\vectorize(\Ib_R).$$

One can then show that
\begin{align*}
\EE[ \Tt\otimes \Tt] 
&=
 \vectorize(\Ib_{d_1}) \circ \cdots\circ \vectorize(\Ib_{d_N}),
\end{align*}
which implies that
\begin{align*}
\EE[y_1^2] =
\langle \EE[\Tt\otimes \Tt],\Xt\otimes \Xt\rangle =
\norm{\Xt}^2_F,
\end{align*}
from which  $\EE\left\|{\fCP{R}(\Xt)}\right\|_2^2    =\left\|{\Xt}\right\|_F^2$  directly follows.

%%%%%%%%%lemma.CP.Vanishingvariance%%%%%%%%%%
\paragraph{Bound on the variance of $\fCP{R}$.}
Similar to TT case, in order to bound the variance of $\norm{\yb}_2^4$ we need to bound $\EE[\left\|{\yb}\right\|_2^4]$. We have
\begin{align*}
\EE[\left\|{\yb}\right\|_2^4]=\sum_{i=1}^k\EE[ y_i^4]+\sum_{i\neq j}\EE [y_i^2y_j^2].
\end{align*} 
Since $y_i$ and $y_j$ are independent whenever $i\neq j$ and $\EE[y_i^2]=\nbr{\Xt}_F^4$ for all $i$, the second summand is equal to $k(k-1)\nbr{\Xt}_F^4$. We now derive a bound on $\EE[y_1^4]$.
First define the tensor $\St^n$ of order $2(n-1)$  and shape $\underbrace{R\times R\dots\times R}_{n-1}\times d_1 \times d_2\cdots\times d_{n-1}$ for any $2\leq n < N$ by

$$
\St^n_{r_1,r_2,\cdots,r_{n-1},i_1,i_2,\cdots,i_{n-1}} = \sum_{r_n,\dots,r_N}\sum_{i_n,\cdots,i_N}(\Ab^n)_{i_nr_n}
(\Ab^{n+1})_{i_{n+1}r_{n+1}}\dots(\Ab^N)_{i_Nr_N}\It{r_1,\dots,r_N}\Xt_{i_1,\dots,i_N},
$$
where $\It \in (\RR^R)^{\otimes N}$ is the $N$th order identity tensor, i.e., $\It_{r_1,\dots,r_n} = 1$ if $r_1=\dots = r_n$ and $0$ otherwise.
In some sense, $\St^n$ is the tensor obtained by removing the first $n-1$ factor matrices from the computation of $y_1 =\langle\CP{\Ab^1,\Ab^2,\cdots,\Ab^N},\Xt\rangle$. With this definition one can check that
\begin{itemize}
    \item $\langle\CP{\Ab^1,\Ab^2,\cdots,\Ab^N},\Xt\rangle = \langle(\Ab^1)^\ts,\Sb^2\rangle,$
     \item $(\St^N_{(1,\dots,N-1)})^\ts = (\Xt_{(N)})^\ts \Ab^N \It_{(1)}$~(recall that $(\St^{N})_{(1,\dots,N-1)}\in\mathbb{R}^{R^{N-1} \times d_1\dots d_{N-1}}$ denotes the matricization of $\St^{N}$ obtained by mapping its first $N-1$ modes to rows and the other ones to columns).
    \item $\vectorize(\St^n) = \left((\St^{n+1})_{(1,2n)}\right)^\ts\vectorize(\Ab^n)$ for each $n \in [N-1].$
   \end{itemize}
 Using Lemma~\ref{Innerproduct} we obtain
 \begin{align*}
 \EE y_1^4 = \EE\langle\CP{\Ab^1,\Ab^2,\cdots,\Ab^N},\Xt\rangle^4 
 &= 
 \EE\langle\vectorize((\Ab^1)^\ts),\vectorize(\Sb^2)\rangle^4 
 = 
 3 R^{-\frac{2}{N}}\EE\left\|{\vectorize(\Sb^2)}\right\|_F^4\\
 &= 
 3 R^{-\frac{2}{N}}\EE\left\|{((\St^3)_{(1,4)})^\ts\vectorize(\Ab^2)}\right\|_F^4.
 \end{align*}
 Using successive applications of Lemma~\ref{Wishartidentity} it  follows that
\begin{align*}
\EE y_1^4  
&=
3 R^{-\frac{2}{N}}\EE\left\|{((\St^3)_{(1,4)})^\ts\vectorize(\Ab^2)}\right\|_F^4
\nonumber \\
&\leq
3^2 R^{-\frac{4}{N}} \EE\left\|(\St^3)_{(1,4)}\right\|_F^4
=
3^2 R^{-\frac{4}{N}} \EE\left\|{\vectorize(\St^3)}\right\|\label{eq13}
=
3^2 R^{-\frac{4}{N}} \EE\left\|{((\St^4)_{(1,6)})^\ts\vectorize(\Ab^3)}\right\|_F^4\nonumber\\
&\leq
3^3 R^{-\frac{6}{N}}
\EE\left\|(\St^4)_{(1,6)}\right\|_F^4
=
3^3 R^{-\frac{6}{N}}
\EE\left\|{\vectorize(\St^4)}\right\|_F^4\nonumber\\
&\leq
\dots\nonumber\\
&\leq
3^{N-1} R^{-\frac{2(N-1)}{N}}\EE\nbr{\vectorize(\St^N)}_F^4=3^{N-1} R^{-\frac{2(N-1)}{N}}\EE\nbr{(\St^N_{(1,\dots,N-1)})^\ts}_F^4\nonumber\\
&=
3^{N-1} R^{-\frac{2(N-1)}{N}}\EE\nbr{(\Xt_{(N)})^\ts \Ab^N \It_{(1)}}_F^4=3^{N-1} R^{-\frac{2(N-1)}{N}}\EE\left\|{(\Xt_{(N)})^\ts \Ab^N}\right\|_F^4\nonumber\\
&\leq
3^{N-1}R^{-2}R(R+2)\nbr{\Xt}_F^4\\
&=
3^{N-1}\left(1+\frac{2}{R}\right)\nbr{\Xt}_F^4\nonumber,
\end{align*}
where we used the equality $\|\Tt \It_{(1)}\|^2_F = \|\Tt\|^2_F$ for any tensor $\Tt$~(which follows from the fact that $\It_{(1)}(\It_{(1)})^\ts = \Ib$) for the penultimate equality.

Similar to proof of Theorem~\ref{variance} for $\fTT{R}$ map, we obtain
\begin{align*}
\EE\nbr{\yb}_2^4=\sum_{i=1}^k\EE y_i^4+\sum_{i\neq j}\EE y_i^2y_j^2
\leq k\left(3^{N-1}\left(1+\frac{2}{R}\right)\nbr{\Xt}_F^4\right)+k(k-1)\left\|{\Xt}\right\|_F^4.
\end{align*}
Finally,
\begin{align*}
\Var{\left\|{\fCP{R}(\Xt)}\right\|_2^2}
&=
\Var{\norm{\frac{1}{\sqrt{k}}\yb}_2^2}=~\frac{1}{k^2}\EE\left(\left\|{\yb}\right\|_2^4\right)-\frac{1}{k^2}\EE\left(\left\|{\yb}\right\|_2^2\right)^2=~\frac{1}{k^2}\EE\left\|{\yb}\right\|_2^4-\left\|{\Xt}\right\|_F^4\\
&\leq
\frac{1}{k^2}\left[k\left(3^{N-1}\left(1+\frac{2}{R}\right)\nbr{\Xt}_F^4\right)+k(k-1)\left\|{\Xt}\right\|_F^4\right]-\left\|{\Xt}\right\|_F^4\\
&\leq
\frac{1}{k}\left(3^{N-1}\left(1+\frac{2}{R}\right)-1\right)\nbr{\Xt}_F^4.
\end{align*}
\end{proof}
\end{theorem*}

\subsection{Proof of Theorem~\ref{JLpropertyCPTT}: CP case}
Theorem~\ref{JLpropertyCPTT} for the map $\fCP{R}$ directly follows from the following concentration bound.
\begin{theorem*}\label{concenterationboundCP}
Let $\Xt\in\RR^{d_1\times d_2 \times \cdots \times d_N}$. There exist absolute constants $C$ and $\widetilde{K}>0$ such that the  random projection map $\fCP{R}$~(see  Definition~\ref{CP.definition}) satisfies
\begin{align*}
\PP\left(\left|\left\|{\fCP{R}(\Xt)}\right\|_2^2 - \left\|{\Xt}\right\|_F^2\right| \geq \varepsilon\left\|{\Xt}\right\|_F^2\right)\leq C\exp\left[-C_1\frac{\left(\sqrt{k}\varepsilon\right)^\frac{1}{N}}{(3^{N-1}\widetilde{K})^{\frac{1}{2N}}(1+2/R)^\frac{1}{2N}}\right].
\end{align*}

%%%%%%%%%%proof.CP.Hypercontractivity%%%%%%%%%%
\begin{proof}
By CP part of Theorem~\ref{variance}, recall
$$\EE\norm{\fCP{R}(\Xt)}_2^2 = \norm{\Xt}_F^2,$$
and
$$\Var{\left\|{\fCP{R}(\Xt)}\right\|_2^2}\leq\frac{1}{k}\left(3^{N-1}\left(1+\frac{2}{R}\right)-1\right)\nbr{\Xt}_F^4.$$
Since $\left\|{\fCP{R}(\Xt)}\right\|_2^2$ is an order $2N$ polynomial of the entries of the matrices $\Ab^1_i,\cdots,\Ab^N_i$ for $i\in [k]$ we can apply Theorem~\ref{thm:hypercontractivity} to obtain
\begin{align*}
\PP\left(\left|\left\|{\fCP{R}(\Xt)}\right\|_2^2 - \left\|{\Xt}\right\|_F^2\right| \geq \lambda \right)
\leq 
C\exp\left[-~\left(\frac{\lambda^2}{\widetilde{K}\Var{\left\|{\fCP{R}(\Xt)}\right\|_2^2}}\right)^\frac{1}{2N}\right],
\end{align*}
where $C = e^2$ and $\widetilde{K}$ are absolute constants. Using the fact that $$\Var{\left\|{\fCP{R}(\Xt)}\right\|_2^2 }\leq \frac{3^{N-1}}{k}(1+2/R)\nbr{\Xt}_F^4,$$
and letting $\lambda = \varepsilon \left\|{\Xt}\right\|_F^2$ we obtain
\begin{align*}
\PP\left(\left|\left\|{\fCP{R}(\Xt)}\right\|_2^2 - \left\|{\Xt}\right\|_F^2\right|\geq \varepsilon \norm{\Xt}_F^2 \right)
&\leq
C\exp\left[-~\left(\frac{k\varepsilon^2\left\|{\Xt}\right\|_F^4}{\widetilde{K}3^{N-1}(1+2/R)\left\|{\Xt}\right\|_F^4}\right)^\frac{1}{2N}\right]\\
&\leq 
C\exp\left[-~\frac{\left(\sqrt{k}\varepsilon\right)^\frac{1}{N}}{(3^{N-1}\widetilde{K})^\frac{1}{2N}(1+2/R)^\frac{1}{2N}}\right].
\end{align*}
\end{proof}
\end{theorem*}

\newpage
\section{Additional Experimental Results}\label{app:xp}
\subsection{Pairwise Distance Estimation}\label{app:xp-cifar}
\begin{figure*}[th!]
\begin{center}
\includegraphics[width=\textwidth]{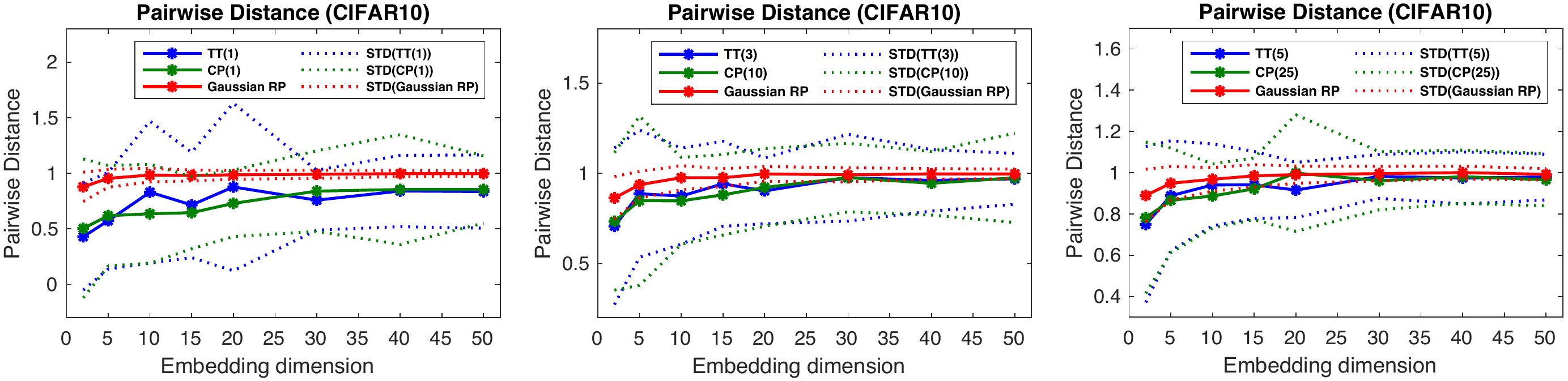}
 \caption{Comparison of tensorized ranodm projections with Gaussian random projections on CIFAR-10 data for different values of the rank parameter: (left) rank 1, (middle) rank 3-10, (right) rank 5-25.}
\label{figure:pairwise}
\end{center}
\end{figure*}
We compare the tensorized projection maps $\fTT{R}$ and $\fCP{R}$ with classical Gaussian RP on \text{CIFAR-10} image data for different values of the rank parameter $R$.   We reshape  the first n=50 vectors~(of size $32\times 32 \times 4$) of \text{CIFAR-10} to $4\times 4 \times 4 \times 4 \times 4\times 3$ tensors, normalize them and compare the pairwise distance $\frac{1}{n(n-1)} \sum_{1\leq i\neq j\leq n} \frac{\norm{f(\xb_i)-f(\xb_j)}_2}{\norm{\xb_i-\xb_j}_2}$ and standard deviation for different projection sizes $k$ over 100 trials. The results are reported in Figure~\ref{figure:pairwise} where we see that tensorized random projection maps perform competitively with classical Gaussian random projections.

\subsection{Time Evaluation}\label{app:xp-time}
\begin{figure*}[th!]
\begin{center}
\includegraphics[scale=0.7]{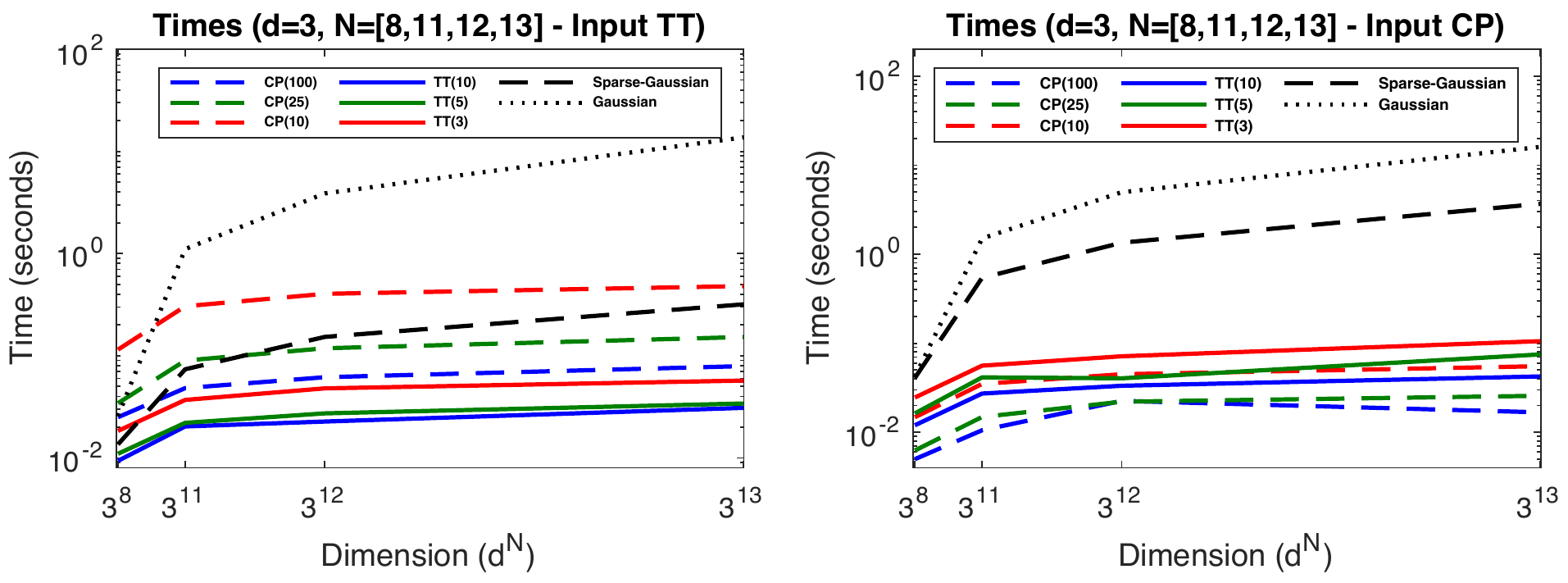}
 \caption{Comparison of embedding time  between tensorized, Gaussian and very  sparse Gaussian  RP  for  the  medium-order case with different number of modes ($d = 3, N \in \{8, 11, 12, 13\}$) when the input is given in the TT format (left) or CP format (right).}
\label{figure:timeevaluation}
\end{center}
\end{figure*}
We report the average running time with respect to the input dimension $d^N$ for the medium-order case with different number of modes $(d = 3, N \in \{8, 11, 12,13\})$ in Figure~\ref{figure:timeevaluation}, when the input tensor $\Xt$ is either as a TT or CP tensor of rank 10. We can see that $\fTT{R}$ is more efficient when the input is in TT format. However, $\fCP{R}$ performs better when the input is in the CP format~(though the computational gain of $\fCP{R}$ in this case is considerably smaller than the one of $\fTT{R}$ in the previous case). We can see that by increasing the dimension $\fTT{R}$ performs close to $\fCP{R}$ even when the input is in CP and it is faster than classical Gaussian RPs in both cases~(which is not true for $\fCP{100})$. 
% Moreover, by increasing the rank $R$, $\fTT{5}$(resp. $\fCP{5}$) and $\fTT{10}$ (resp. $\fCP{5}$) are very close when the input is in TT (resp. for input in CP). 

%

%

\end{document}